\documentclass[11pt]{article}

\usepackage[preprint]{acl}

\usepackage[T1]{fontenc}
\usepackage[utf8]{inputenc}
\usepackage{times}
\usepackage{latexsym}
\usepackage{microtype}
\usepackage{amsmath,amssymb,amsthm,mathtools}
\usepackage{booktabs}
\usepackage{multirow}
\usepackage{array}
\usepackage{tabularx}
\usepackage{graphicx}
\usepackage{xcolor}
\usepackage{algorithm}
\usepackage{algorithmic}
\usepackage{tikz}
\usetikzlibrary{positioning,arrows.meta,calc}
\usepackage[nameinlink,capitalize]{cleveref}

\newcommand{\vocab}{\mathcal{V}}
\newcommand{\docs}{\mathcal{D}}
\newcommand{\simplex}{\Delta_V}
\newcommand{\R}{\mathbb{R}}

\newcommand{\KL}{\mathrm{KL}}
\newcommand{\JS}{\mathrm{JS}}

\newcommand{\logsumexp}{\mathrm{logsumexp}}
\newcommand{\pfull}{p_{\mathrm{full}}}
\newcommand{\pret}{p_{\mathrm{ret}}}

\newcommand{\qgd}{q_{\mathrm{GD}}}

\newcommand{\ci}[1]{$_{\pm #1}$}
\newcommand{\prett}{p_{\mathrm{ret},t}}
\newcommand{\pfullat}[1]{p_{\mathrm{full},#1}}

\newtheorem{theorem}{Theorem}

\title{Grounded Decoding: Retrieval-Anchored Probability Fusion for Faithful RAG}

\author{
\textbf{Ibne Farabi Shihab}\textsuperscript{1}
\and
\textbf{Fariya Afrin}\textsuperscript{2}
\and
\textbf{Sanjeda Akter}\textsuperscript{1}
\and
\textbf{Anuj Sharma}\textsuperscript{3}
\\[2pt]
\textsuperscript{1}Department of Computer Science, Iowa State University \\
\textsuperscript{2}Department of Computer Science, Kalinga Institute of Industrial Technology \\
\textsuperscript{3}Department of Civil, Construction \& Environmental Engineering, Iowa State University \\
\texttt{ishihab@iastate.edu}
}

\date{}

\begin{document}
\maketitle

\begin{abstract}
As retrieval-augmented generation (RAG) systems scale, it becomes increasingly challenging to ensure faithful grounding in external evidence. Large language models may still prioritize parametric knowledge over retrieved information when conflicts arise. We propose a novel training-free decoding framework, \emph{Grounded Decoding}, designed to improve factual consistency in RAG without modifying model parameters.
Unlike standard approaches that rely on a single conditional distribution, our method constructs two matched-prompt distributions at every generation step: (1) a full RAG distribution conditioned on the query, retrieved documents, and generated prefix, and (2) a retrieval-only distribution conditioned solely on retrieved evidence and the same prefix. The final next-token distribution is derived as the unique solution to a KL-barycenter objective over the probability simplex, yielding a normalized geometric fusion of the two distributions.This formulation naturally recovers standard RAG when the grounding weight is zero and smoothly shifts probability mass toward retrieved evidence as grounding strength increases. We further introduce a conflict-aware adaptive weighting scheme that dynamically adjusts grounding based on distributional disagreement and retriever confidence.
Experiments on ALCE, Natural Questions, and FActScore demonstrate consistent improvements in factual accuracy and citation quality over standard RAG and competitive decoding-time baselines, while maintaining fluency. Our results indicate that probability-level fusion provides a strong and efficient alternative to logit-level intervention methods for faithful RAG decoding.
\end{abstract}

\section{Introduction}
\label{sec:introduction}

Retrieval-augmented generation (RAG) has become a standard mechanism for connecting large language models to external knowledge sources \citep{lewis2020retrieval,guu2020realm,borgeaud2022improving}, yet these systems still struggle to ensure faithful grounding in retrieved evidence. While retrievers supply passages that can be inspected and cited, and generators transform them into fluent responses, this separation does not guarantee that generated outputs remain fully supported by the provided context. The generator may still rely on parametric knowledge, ignore relevant evidence, or produce statements that are plausible but not grounded in the retrieved documents \citep{ji2023survey,huang2025survey}. These limitations are particularly evident in long-form question answering and biography generation, where responses contain multiple atomic facts and only a subset is explicitly supported by the retrieved corpus \citep{gao2023enabling,min2023factscore}.

Recent work has explored decoding-time interventions as a lightweight alternative to retraining large language models. Context-Aware Decoding (CAD) introduces a contrast between context-conditioned and context-free distributions, amplifying tokens that depend on external evidence \citep{shi2024trusting}. Several adaptive variants, including AdaCAD, COIECD, and CoCoA, further refine this mechanism by adjusting contrast strength based on contextual reliability \citep{wang2025adacad,yuan2024coiecd,khandelwal2025cocoa}. While these approaches are effective in practice, they are typically formulated as logit-level manipulations, which makes the induced probability distribution difficult to interpret and complicates analysis of boundary behavior and implementation under modern caching-based inference systems.

We revisit grounded decoding for RAG from the perspective of probability fusion on the probability simplex. Instead of performing logit subtraction or contrastive scaling, we construct a full RAG distribution conditioned on the query, retrieved documents, and prefix, together with a retrieval-only distribution computed under a matched prompt. The resulting decoding objective is formulated as a KL barycenter between the two distributions, which admits a closed-form solution as a normalized geometric mixture (logarithmic opinion pool). This formulation provides a clean interpretation: the method recovers standard RAG when the grounding weight is zero, and smoothly interpolates toward the retrieval-only distribution as grounding strength increases. Unlike linear contrastive formulations in logit space, this objective avoids double-counting effects and yields stable behavior across limiting regimes.

Beyond this formulation, several practical considerations arise in implementing grounded decoding. The retrieval-only stream must be evaluated using the same autoregressive prefix as the full RAG stream, which requires a synchronized dual-stream decoding process. As a result, both streams evolve token-by-token and maintain separate key-value caches rather than relying on precomputed retrieval representations. This structure is essential for correctness but is often omitted in prior contrastive decoding formulations.

We further introduce an adaptive grounding mechanism that dynamically adjusts the fusion strength based on distributional disagreement and retriever confidence. Specifically, the grounding weight increases when the Jensen--Shannon divergence between the full and retrieval-only distributions is high, indicating conflict between parametric and external evidence, and decreases when retrieval confidence is low to prevent over-reliance on noisy evidence. This design couples the optimization objective with the decoding dynamics, enabling stronger grounding on factual tokens while preserving fluency in low-confidence regions.

We evaluate Grounded Decoding against standard RAG, CAD, adaptive CAD variants, DoLa, and kNN-LM under matched decoding conditions. Experiments on ALCE, Natural Questions, and FActScore demonstrate consistent improvements in factuality and citation quality while maintaining competitive fluency. Ablation studies show that the retrieval-only distribution is the primary driver of gains, that the KL barycenter formulation consistently outperforms linear mixtures, and that the adaptive grounding mechanism is predominantly activated on factual spans rather than uniformly across generated outputs.

Our contributions include:
\begin{itemize}
    \item A probability-fusion framework for RAG decoding based on a KL barycenter formulation
    \item A closed-form normalized geometric mixture between full RAG and retrieval-only distributions
    \item A synchronized dual-stream decoding mechanism for consistent autoregressive evaluation
    \item An adaptive grounding strategy based on distributional disagreement and retriever confidence
    \item Extensive empirical validation showing improved factuality and citation quality across benchmarks
\end{itemize}
Our findings reveal that probability-level fusion provides a principled mechanism for improving faithfulness in RAG decoding, with retrieval-only grounding playing a more critical role than full-context generation in resolving factual inconsistencies. These results improve the reliability of grounded generation and deepen our understanding of probabilistic evidence integration in large language model decoding.

\section{Related Work}
\label{sec:related}

Retrieval-augmented generation (RAG) connects parametric language models with non-parametric memory or external documents \citep{lewis2020retrieval,guu2020realm,borgeaud2022improving}. In open-domain question answering and long-form generation, factual correctness alone is insufficient, as outputs must be grounded in the provided evidence. This has led to increased emphasis on attribution and citation quality as key evaluation dimensions \citep{gao2023enabling,min2023factscore,kwiatkowski2019natural}.

\paragraph{Training-time and Decoding-time Methods}
Training-time approaches address this issue through instruction tuning, preference optimization, and self-refinement \citep{ouyang2022training,rafailov2023direct,asai2024selfrag}. While effective, these methods require additional data and training compute and are less flexible once deployed. Decoding-time methods provide a complementary alternative, as they operate on frozen models and can be integrated with improved retrievers or verification modules.

\paragraph{Context-Aware Decoding}
Within this family, Context-Aware Decoding (CAD) introduces a contrastive mechanism that amplifies tokens whose probabilities depend on the provided context \citep{shi2024trusting}. Adaptive variants such as AdaCAD, COIECD, and CoCoA extend this idea by modeling context uncertainty or model disagreement and adjusting the contrastive strength dynamically \citep{wang2025adacad,yuan2024coiecd,khandelwal2025cocoa}. However, these methods are typically implemented at the logit level, which makes the induced probability distributions less explicit and limits direct interpretability.

A more comprehensive discussion of Grounded Decoding, Decoding-Time Interventions, Controllable Generation Methods, and Retrieval-Grounding approaches is provided in Appendix~\ref{app:related_work}.
\section{Methodology}
\subsection{Problem Formulation and Background}
\label{sec:background}
Let \(x\) denote a query and \(\docs=\{d_1,\ldots,d_m\}\) the documents returned by a retriever. A RAG model generates \(y=(y_1,\ldots,y_T)\) autoregressively from a prompt containing both \(x\) and \(\docs\). At decoding step \(t\), the generated prefix is \(y_{<t}\), the vocabulary is \(\vocab\) with \(|\vocab|=V\), and the probability simplex is
\begin{equation}
\simplex = \{q\in\R^V : q_i\geq 0,\ \sum_{i=1}^V q_i=1\}.
\end{equation}

The standard full RAG distribution is defined as
\begin{equation}
{\pfull}_t(i)=p_\theta(i\mid x,\docs,y_{<t}),
\label{eq:pfull}
\end{equation}
which serves as the base distribution for conventional decoding strategies, including greedy decoding, nucleus sampling, and beam search, prior to any grounding-specific modification.

The primary limitation of this distribution is that it is not explicitly constrained to be evidence-aware. Although retrieval provides access to relevant documents, it does not enforce alignment between next-token probabilities and the retrieved evidence. As a result, the model may default to memorized parametric knowledge, interpolate across inconsistent passages, or produce connective statements that are not supported by the evidence. This issue becomes particularly pronounced in knowledge-conflict settings, where retrieved documents may deliberately contradict the model’s parametric memory \citep{longpre2021entity}. A faithful decoder should therefore respond to such conflicts in a local and context-sensitive manner, rather than uniformly modifying all tokens, by increasing reliance on external evidence when disagreement arises between the full RAG distribution and an evidence-conditioned reference.

This motivation has led to a family of contrastive decoding approaches. Context-Aware Decoding (CAD) introduces a logit-level contrast between a full-context model and a context-removed variant \citep{shi2024trusting}. In log-probability form, this can be expressed as
\begin{align}
\log q_{\mathrm{CAD},t}(i) \propto{} &(1+\beta)\log p_\theta(i\mid x,\docs,y_{<t}) \notag\\
&{}-\beta\log p_\theta(i\mid x,y_{<t}),
\label{eq:cad}
\end{align}
where \(\beta\) controls the degree of contextual amplification. Subsequent adaptive variants replace the fixed weighting parameter with signals derived from distributional disagreement or uncertainty estimates \citep{wang2025adacad,yuan2024coiecd,khandelwal2025cocoa}. While effective, these methods operate primarily at the logit level, which obscures the resulting probability distribution and complicates direct interpretation.

Grounded Decoding belongs to the same general class of contrastive methods and can similarly be implemented as a post-hoc transformation of the model’s output distribution. However, it is not a direct reformulation of CAD. In contrast to CAD, which suppresses tokens that are likely in the absence of context, Grounded Decoding explicitly promotes tokens that remain probable under the retrieved documents themselves. This distinction leads to different optimization behavior and empirical effects, particularly in settings requiring fine-grained evidence alignment.

A useful conceptual perspective arises from recent work framing decoding as optimization over the probability simplex \(\simplex\), where regularization terms encode preferences such as entropy control, sparsity, or truncation behavior \citep{ji2026decoding}. From this viewpoint, a central question in faithful RAG decoding is the choice of reference distribution and the corresponding regularizer that governs boundary behavior. The following section addresses this question by introducing a per-token retrieval reference and formulating decoding as a KL-barycenter problem.

\begin{figure*}[t]
\centering
\includegraphics[width=0.9\textwidth]{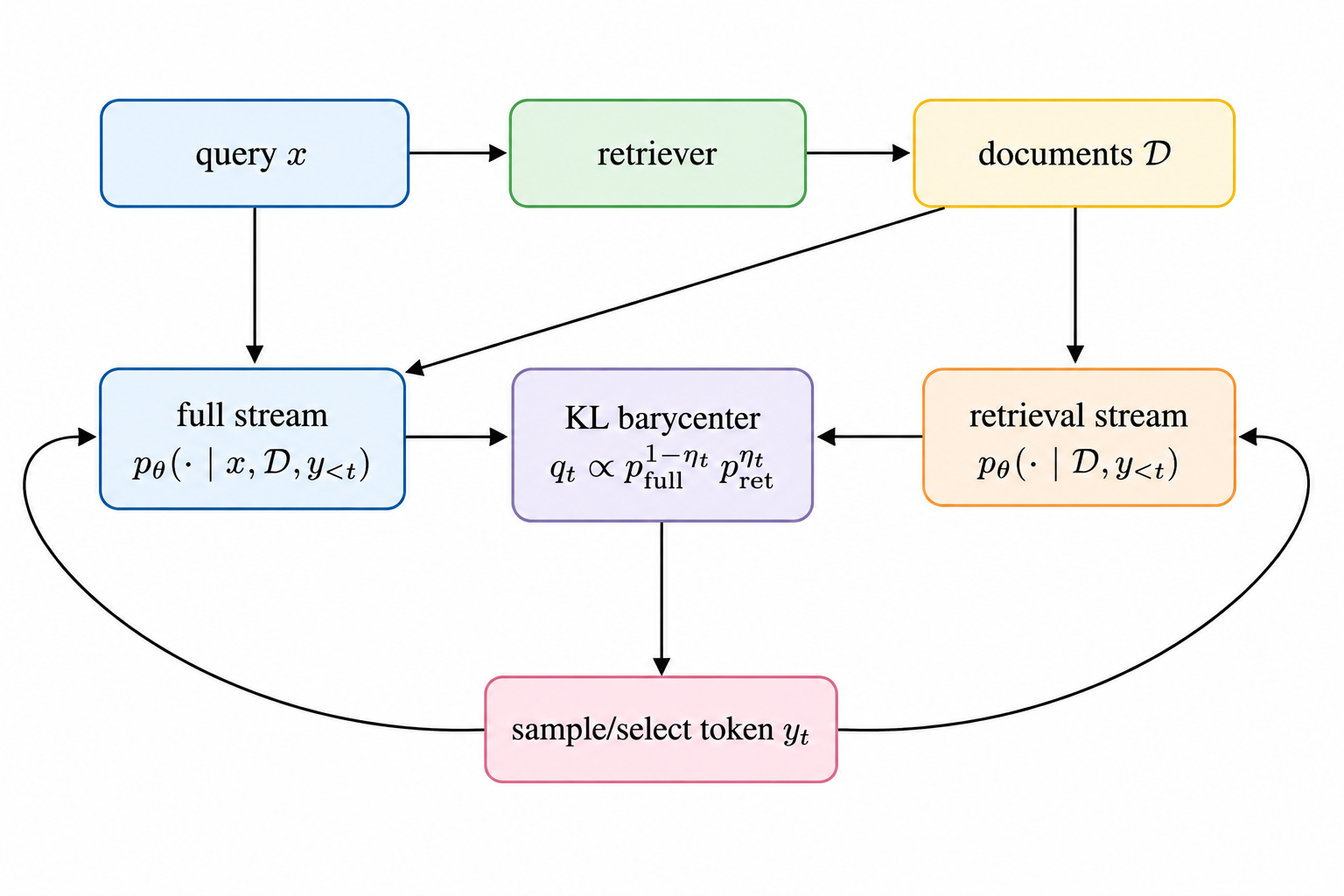}
\caption{Grounded Decoding maintains two matched decoding streams. The full stream conditions on the query and retrieved documents, while the retrieval stream conditions only on the retrieved documents and the same generated prefix. The final fused distribution is computed as a KL barycenter with an adaptive weighting scheme.}
\label{fig:grounded_decoding}
\end{figure*}
\subsection{Grounded Decoding}
\label{sec:method}

\Cref{fig:grounded_decoding} summarizes the proposed dual-stream decoding framework, including the full RAG stream, the retrieval-only stream, and the KL-barycenter fusion step.
\paragraph{A per-token retrieval reference}
\label{sec:retrieval_reference}

We define a retrieval-conditioned reference distribution using the same language model and the same generated prefix as the full RAG stream, but under a prompt that includes only the retrieved documents and omits the user query:
\begin{equation}
{\pret}_t(i)=p_\theta(i\mid \docs,y_{<t}).
\label{eq:pret}
\end{equation}

Since this distribution is computed using the model’s own output head, it remains dense over the full vocabulary, is fully compatible with the model’s tokenization scheme, and avoids calibration inconsistencies that may arise when combining heterogeneous probability sources. Importantly, the dependence on \(y_{<t}\) is essential: the distribution \(\prett\) varies with the generated prefix, and therefore precomputing \(p_\theta(\cdot\mid \docs)\) independently of generation is not valid.

In practice, the document component of the prompt can be prefetched into a retrieval-stream key-value cache; however, each generated token still requires a forward pass conditioned on this cache.

A consistent prompt formulation is also critical. For instruction-tuned models, the full and retrieval streams must share the same chat template, system message, document formatting, and assistant prefix, differing only in the inclusion or exclusion of the query. Otherwise, discrepancies between \({\pfull}_t\) and \({\pret}_t\) may reflect prompt-format artifacts rather than genuine evidence sensitivity, leading to an incorrect characterization of distributional differences. \Cref{app:prompts} provides the exact prompt templates used in our implementation, and \Cref{fig:prompt_template_comparison} visualizes the controlled difference between the full and retrieval-only streams.

\paragraph{KL barycenter objective and closed form}
\label{sec:barycenter}

We aim to construct a fused next-token distribution that remains faithful to the full RAG distribution while incorporating constraints from retrieved evidence. In particular, the decoder should preserve standard query-conditioned behavior while simultaneously assigning higher probability mass to tokens supported by external documents. This trade-off can be naturally formulated as a KL barycenter. For a grounding weight \(\rho_t \geq 0\), Grounded Decoding is defined as:
\begin{equation}
{\qgd}_t = \arg\min_{q\in\simplex}
\left[\KL(q\,\|\,{\pfull}_t)+\rho_t\KL(q\,\|\,{\pret}_t)\right].
\label{eq:main_objective}
\end{equation}

The first term penalizes departure from standard RAG; the second term pulls the distribution toward tokens that remain likely under the retrieved evidence. The scalar \(\rho_t\) controls how much retrieval anchoring is applied at this token, and we will allow it to vary with the input and the decoding state in \cref{sec:adaptive}.

The objective admits a unique closed-form solution that is straightforward to compute.

\begin{theorem}[Closed-form grounded distribution]
\label{thm:closed_form}
Assume \({\pfull}_t(i)>0\) and \({\pret}_t(i)>0\) for all \(i\in\vocab\), and let \(\rho_t\geq0\). The unique minimizer of \cref{eq:main_objective} is
\begin{align}
{\qgd}_t(i)
&=\frac{{\pfull}_t(i)^{1/(1+\rho_t)}{\pret}_t(i)^{\rho_t/(1+\rho_t)}}{Z_t},
\label{eq:closed_form}\\
Z_t
&=\sum_{j\in\vocab}{\pfull}_t(j)^{1/(1+\rho_t)}
{\pret}_t(j)^{\rho_t/(1+\rho_t)}.
\notag
\end{align}
Equivalently, with \(\eta_t=\rho_t/(1+\rho_t)\in[0,1)\),
\begin{align}
\log {\qgd}_t(i)
&=(1-\eta_t)\log{\pfull}_t(i)\notag\\
&\quad+\eta_t\log{\pret}_t(i)-A_t,
\label{eq:log_pool}
\end{align}
where \(A_t\) is the log normalizer.
\end{theorem}

\begin{proof}
Expanding the objective gives
\begin{align}
&\KL\!\left(q\middle\|\pfullat{t}\right)
+\rho_t\KL\!\left(q\middle\|\prett\right)
\notag\\
&\quad =
(1+\rho_t)\sum_i q_i\log q_i
-\sum_i q_i\log \pfullat{t}(i)
\notag\\
&\qquad
-\rho_t\sum_i q_i\log \prett(i)
+ C .
\end{align}
where \(C\) collects entropy-of-reference constants that do not affect the minimizer. Introducing a Lagrange multiplier \(\mu\) for the constraint \(\sum_i q_i=1\), the first-order condition at any coordinate in the relative interior of the simplex is
\begin{align}
(1+\rho_t)(\log q_i+1)&-\log{\pfull}_t(i)\notag\\
&-\rho_t\log{\pret}_t(i)+\mu=0.
\end{align}
Solving for \(q_i\) and normalizing over the vocabulary yields \cref{eq:closed_form}. Strict convexity of the objective follows because the negative entropy term has positive coefficient \(1+\rho_t\) while the reference-dependent terms are linear in \(q\), so the optimum is unique. Full support of the two softmax distributions ensures that the optimum lies in the relative interior in exact arithmetic; \cref{app:proof_details} discusses numerical considerations.
\end{proof}

The closed form makes the boundary behavior of the method easy to read off. When \(\rho_t=0\), the exponent \(\eta_t\) is also zero and \({\qgd}_t\) collapses exactly onto \({\pfull}_t\); the method is therefore an exact extension of standard RAG rather than a perturbation that must be tuned toward zero to recover the baseline. As \(\rho_t\to\infty\), \(\eta_t\to1\) and \({\qgd}_t\to{\pret}_t\), so very large grounding weights anchor the decoder to the retrieval reference. Intermediate values interpolate smoothly between the two. The log-odds between any two tokens \(i\) and \(j\) take the equally simple form
\begin{align}
\log\frac{{\qgd}_t(i)}{{\qgd}_t(j)}
&=(1-\eta_t)\log\frac{{\pfull}_t(i)}{{\pfull}_t(j)}\notag\\
&\quad+\eta_t\log\frac{{\pret}_t(i)}{{\pret}_t(j)},
\label{eq:log_odds}
\end{align}
which shows that the fused distribution prefers token \(i\) over token \(j\) whenever both streams prefer \(i\), but reweights more cautiously when the two streams disagree. The update is invariant to additive shifts in logits because it is written in log probabilities, and its scale is controlled by \(\rho_t\) rather than by the arbitrary magnitude of raw model logits.

Finally, the KL-barycenter formulation avoids degeneracies that may arise in linear logit interpolation with KL regularization. In such cases, the objective can implicitly double-count the full-context distribution, leading to non-monotonic behavior in the regularization strength. In contrast, the barycentric formulation ensures a direct and monotone interpolation between the two distributions. A detailed comparison is provided in \cref{app:linear_logit}; \Cref{fig:kl_fusion,fig:grounded_simplex} visualize the KL-barycenter interpolation and its contrast with arithmetic probability mixing.

\subsection{Adaptive Grounding}
\label{sec:adaptive}

A fixed grounding weight is suboptimal in retrieval-augmented decoding, as it is often insufficient when the model prior conflicts with external evidence and overly restrictive when retrieved passages are noisy or irrelevant. The closed-form solution in \cref{eq:closed_form} suggests a natural balancing principle: the grounding weight $\rho_t$ should increase under stronger disagreement between the two decoding streams and decrease when retrieval reliability is low. Accordingly, we define
\begin{equation}
\rho_t = \rho_{\max}\, r(x,\docs)\, c_t,
\label{eq:rho_t}
\end{equation}
where $\rho_{\max}$ is a scalar hyperparameter tuned on a validation split under the same budget as CAD, AdaCAD, COIECD, and CoCoA. The terms $c_t$ and $r(x,\docs)$ capture token-level conflict and instance-level retrieval trust, respectively. \Cref{fig:adaptive_grounding} illustrates this computation graphically, showing how token-level conflict and retrieval trust jointly determine the adaptive grounding weight.

The token-level conflict score quantifies local disagreement between the full RAG distribution and the retrieval-conditioned distribution. We define a normalized Jensen--Shannon divergence:
\begin{align}
c_t
&= \frac{\JS\!\left(\pfullat{t}, \prett\right)}{\log 2},
\label{eq:conflict_score}\\
\JS(p,q)
&= \frac{1}{2}\KL(p\|m)
 + \frac{1}{2}\KL(q\|m),
\notag\\
m
&= \frac{p+q}{2}.
\notag
\end{align}
which is bounded in $[0,1]$ for finite vocabularies under natural logarithms. Higher disagreement leads to increased grounding strength by amplifying the contribution of the retrieval distribution in \cref{eq:closed_form}, consistent with the underlying optimization objective. Compared to KL divergence, JS divergence is symmetric and bounded, providing improved stability for adaptive control even in cases of highly skewed distributions.

The instance-level retrieval trust score is designed to mitigate over-grounding when retrieval quality is unreliable. Let $s_k$ denote the relevance score associated with passage $d_k$. We compute normalized weights as
\begin{equation}
 a_k = \frac{\exp(s_k/\tau_r)}{\sum_{\ell=1}^{m}\exp(s_\ell/\tau_r)},
\label{eq:retriever_weights}
\end{equation}
where $\tau_r$ is a temperature parameter, set on the validation split or fixed to one when retriever scores are already calibrated. The aggregated retrieval score is defined as $\bar{s} = \sum_k a_k s_k$, which is mapped to a bounded trust signal via
\begin{equation}
 r(x,\docs) = \sigma\big(\gamma(\bar{s} - b)\big),
\label{eq:trust_score}
\end{equation}
where $b$ is the median top-document score on the validation set and $\gamma$ controls the slope of the gating function. This signal does not aim to verify factual correctness; instead, it serves as a conservative mechanism that reduces grounding strength under low-confidence retrieval conditions. The score is computed once per instance, ensuring that the additional computational overhead remains negligible relative to per-token operations.

For ablation purposes, we also evaluate a static variant in which $\rho_t \equiv \rho$, thereby removing adaptive modulation. This setting isolates the effect of the KL-barycenter objective from the proposed scheduling strategy. Empirically, the static variant already yields consistent improvements over standard RAG decoding, while the adaptive formulation provides additional gains, particularly in high-conflict or low-retrieval-quality settings.

\Cref{fig:grounded_decoding} provides an overview of the proposed framework and \cref{alg:grounded_decoding} formalizes its algorithmic implementation. The complete procedure is presented in Algorithm~\ref{alg:grounded_decoding}, with detailed implementation specifics and the full algorithmic description deferred to Appendix~\ref{app:implementation_details}.

\section{Experimental Setup}
\label{sec:experiments}

We evaluate Grounded Decoding under a matched decoding protocol designed to isolate the effect of the decoding strategy from other confounding factors. All methods are evaluated under identical settings, including retrieved passages, prompt templates, citation format, sampling temperature, truncation strategy, stopping criteria, and maximum generation length. Hyperparameters are selected on a validation split with an equal tuning budget across methods, and all reported results are computed on a held-out test set. We report paired bootstrap 95\% confidence intervals over test instances using 10,000 resamples, with all metrics computed over the same generated outputs for fair comparison.

The base model is Llama-3-8B-Instruct \citep{meta2024llama3}, and the retriever is Contriever \citep{izacard2022unsupervised}, which indexes a fixed Wikipedia snapshot across all knowledge-intensive tasks. All experiments are conducted on NVIDIA A100 80GB GPUs using a vLLM backend modified to support dual-stream probability fusion. We release the full experimental setup, including model checkpoints, tokenizer, chat templates, retrieval index snapshot, CUDA version, vLLM commit, and random seeds, to ensure full reproducibility.

We evaluate on standard benchmarks including ALCE (ASQA, QAMPARI, and ELI5), reporting correctness, citation precision, citation recall, citation F1, and MAUVE \citep{gao2023enabling}; Natural Questions, reporting exact match, token-level F1, and evidence support rate \citep{kwiatkowski2019natural}; and FActScore, reporting atomic factual precision, along with average number of atomic facts and average generation length \citep{min2023factscore}. Reporting both length and fact counts alongside factuality metrics is necessary to ensure that improvements are not driven by overly conservative or overly verbose generation behavior.

The baseline set includes standard RAG using the same retriever and matched decoding configuration; CAD with a validation-tuned \(\beta\); AdaCAD, COIECD, and CoCoA with their published adaptive mechanisms under an equal tuning budget; DoLa with a validation-selected layer configuration; and kNN-LM with a fixed retrieval datastore. \Cref{app:baseline_config} provides the full tuning grid and implementation details for all baselines.
\section{Results}
\label{sec:results}

\begin{table*}[t]
\centering
\small
\caption{Main evaluation results under the matched decoding protocol. Grounded Decoding is the strongest method among the evaluated decoding-time baselines on factuality and citation metrics while maintaining comparable MAUVE fluency. Subscripts denote 95\% paired bootstrap confidence intervals. Best results are in bold.}
\label{tab:main_results}
\setlength{\tabcolsep}{2.7pt}
\renewcommand{\arraystretch}{1.08}
\resizebox{\textwidth}{!}{%
\begin{tabular}{lccccccccc}
\toprule
\textbf{Method} & \textbf{NQ EM} & \textbf{NQ F1} & \textbf{Support} & \textbf{ALCE Correct} & \textbf{Cit. P} &\textbf{ Cit. R} & \textbf{Cit. F1} & \textbf{MAUVE} &\textbf{ FActScore} \\
\midrule
Standard RAG & 41.2\ci{1.1} & 49.3\ci{1.0} & 62.1\ci{1.2} & 42.5\ci{1.0} & 45.1\ci{1.0} & 50.2\ci{1.1} & 47.5\ci{0.9} & 0.88\ci{0.01} & 71.2\ci{1.5} \\
CAD & 42.1\ci{1.2} & 50.4\ci{1.1} & 65.4\ci{1.3} & 44.2\ci{1.2} & 48.3\ci{1.1} & 51.5\ci{1.1} & 49.8\ci{1.1} & 0.85\ci{0.02} & 74.5\ci{1.4} \\
AdaCAD & 44.3\ci{1.0} & 52.1\ci{0.9} & 68.2\ci{1.1} & 46.5\ci{1.1} & 52.1\ci{1.0} & 53.4\ci{1.0} & 52.7\ci{1.0} & 0.86\ci{0.02} & 78.1\ci{1.2} \\
COIECD & 44.8\ci{1.1} & 52.5\ci{1.0} & 69.1\ci{1.0} & 46.8\ci{1.0} & 52.8\ci{1.0} & 54.1\ci{1.0} & 53.4\ci{0.9} & 0.87\ci{0.01} & 78.9\ci{1.1} \\
CoCoA & 45.1\ci{1.0} & 52.8\ci{0.9} & 69.8\ci{1.2} & 47.2\ci{0.9} & 53.5\ci{1.0} & 54.8\ci{1.0} & 54.1\ci{1.0} & 0.87\ci{0.01} & 79.5\ci{1.2} \\
DoLa & 42.5\ci{1.2} & 50.1\ci{1.1} & 64.2\ci{1.1} & 43.8\ci{1.1} & 47.1\ci{1.1} & 50.1\ci{1.1} & 48.5\ci{1.1} & 0.86\ci{0.02} & 73.2\ci{1.4} \\
kNN-LM & 41.8\ci{1.1} & 49.8\ci{1.0} & 63.5\ci{1.3} & 43.1\ci{1.0} & 46.5\ci{1.0} & 50.8\ci{1.1} & 48.5\ci{1.0} & 0.88\ci{0.01} & 72.5\ci{1.5} \\
Grounded Decoding (static) & 45.8\ci{0.9} & 53.5\ci{0.8} & 71.5\ci{1.0} & 48.5\ci{0.8} & 55.2\ci{0.9} & 56.5\ci{0.9} & 55.8\ci{0.9} & \textbf{0.89\ci{0.01}} & 81.2\ci{1.0} \\
Grounded Decoding (adaptive) & \textbf{47.6\ci{0.8}} & \textbf{55.2\ci{0.8}} & \textbf{74.8\ci{0.9}} & \textbf{50.2\ci{0.8}} & \textbf{58.4\ci{0.8}} & \textbf{58.9\ci{0.9}} & \textbf{58.6\ci{0.8}} & \textbf{0.89\ci{0.01}} & \textbf{84.5\ci{0.9}} \\
\bottomrule
\end{tabular}%
}
\end{table*}
\paragraph{Aggregate Performance Comparison}
\label{sec:main_results}

The aggregate performance is presented in \cref{tab:main_results}. The adaptive variant of Grounded Decoding achieves the highest FActScore, evidence support rate, and citation F1 among the evaluated decoding-time methods, while maintaining MAUVE comparable to standard RAG. The static variant also improves over standard RAG and all evaluated baselines, indicating that the KL-barycenter objective with a retrieval reference contributes a substantial portion of the overall gain. The adaptive schedule provides additional consistent improvements by modulating grounding strength at high-conflict tokens.

The results align with the design motivation of the method: the static formulation already biases generation toward evidence-consistent tokens, while the adaptive mechanism further concentrates this effect on tokens where retrieval-model disagreement is most pronounced.The aggregate trends in factuality, citation quality, and fluency are visualized in \Cref{fig:factuality_metrics,fig:citation_metrics,fig:mauve_metrics}.
\begin{table}[t]
\centering
\small
\caption{Ablations. Each row changes one component relative to full adaptive Grounded Decoding. Average length and average atomic facts help rule out gains caused only by shorter generations.}
\label{tab:ablations}
\resizebox{\columnwidth}{!}{%
\begin{tabular}{lcccc}
\toprule
\textbf{Variant} & \textbf{FActScore} & \textbf{Cit. F1} &\textbf{ Avg.\ facts }& \textbf{Avg.\ tokens} \\
\midrule
Standard RAG, $\rho_t=0$ & 71.2 & 47.5 & 24.5 & 185 \\
Static $\rho$ & 81.2 & 55.8 & 23.8 & 180 \\
Conflict only, $r(x,\docs)=1$ & 82.1 & 56.5 & 23.5 & 178 \\
Trust only, $c_t=1$ & 81.8 & 56.2 & 24.0 & 181 \\
Query-only reference & 76.5 & 51.2 & 24.2 & 183 \\
Arithmetic mixture & 82.5 & 56.8 & 23.6 & 179 \\
Full adaptive Grounded & \textbf{84.5} & \textbf{58.6} & 23.4 & 176 \\
\bottomrule
\end{tabular}%
}
\end{table}
\paragraph{Ablation Study}
\label{sec:ablations}

The ablation results in \cref{tab:ablations} disentangle the contributions of individual components. We replace the retrieval-only reference with a query-only reference (closer in spirit to CAD) and observe a notable drop in performance, with FActScore decreasing by nearly five points and citation F1 decreasing by more than four points. This result confirms that the retrieval-conditioned stream serves as the primary source of improvement. The arithmetic interpolation \((1-\alpha){\pfull}_t+\alpha{\pret}_t\) remains competitive, which shows that the existence of a retrieval-only reference matters more than the exact fusion form; however, the KL barycenter consistently performs better, which suggests more effective suppression of tokens deemed implausible by either stream. We ablate either the conflict-based factor or the retrieval-trust factor from the adaptive schedule and observe degraded performance relative to the full model, although both variants still outperform the static baseline. This result indicates that both components contribute complementary signals. We further analyze average facts and token counts, which rule out length effects and confirm that gains do not arise from shorter or more conservative generations. Supplementary ablation visualizations are provided in \Cref{fig:ablation_quality_metrics,fig:ablation_avg_tokens,fig:ablation_avg_facts}.
\begin{table}[t]
\centering
\small
\caption{Token-level trace from a representative FActScore biography. Adaptive grounding activates on factual spans (year, occupation, geographic relation) and remains close to zero on function words.}
\label{tab:trace}
\begin{tabular}{lccc}
\toprule
\textbf{Generated token} & \textbf{$\JS$ }&\textbf{ $\rho_t$} &\textbf{ $\eta_t$} \\
\midrule
\texttt{The}      & 0.02 & 0.06 & 0.06 \\
\texttt{novelist} & 0.41 & 1.27 & 0.56 \\
\texttt{was}      & 0.03 & 0.09 & 0.08 \\
\texttt{born}     & 0.04 & 0.13 & 0.12 \\
\texttt{in}       & 0.05 & 0.16 & 0.14 \\
\texttt{1947}     & 0.62 & 1.92 & 0.66 \\
\texttt{in}       & 0.04 & 0.13 & 0.11 \\
\texttt{Trieste}  & 0.58 & 1.80 & 0.64 \\
\bottomrule
\end{tabular}
\end{table}
\paragraph{Token-Level Analysis}
\label{sec:token_analysis}

To ground the analysis at the token level, \cref{tab:trace} presents a representative decoding trace from the FActScore evaluation set. The conflict score and grounding weight remain low for function words but increase sharply at tokens corresponding to factual entities such as years, occupations, and geographic names, which are known to be frequent sources of hallucination in long-form generation. We record these quantities alongside \({\pfull}_t(y_t)\), \({\pret}_t(y_t)\), and \({\qgd}_t(y_t)\) for all evaluated instances in the released artifact. These traces serve as behavioral diagnostics of the decoding dynamics rather than as evidence of internal causal mechanisms.\Cref{fig:token_level_grounding,fig:adaptive_grounding_trajectory} provide complementary visualizations of the same token-level grounding behavior.
\paragraph{Sensitivity Analysis}
\label{sec:sensitivity}

Sensitivity to the grounding weight remains limited across a broad validation range. For small values of \(\rho_{\max}\), the behavior closely resembles standard RAG, as \(\rho_t\) remains near zero for most decoding steps. In contrast, very large values over-anchor generation to the retrieval stream and may reduce answer specificity, particularly in ELI5-style settings where responses require additional explanatory content beyond retrieved facts. The selected operating point balances improvements in attribution quality against potential degradation in linguistic richness. Full sweep curves with confidence intervals are provided in the released artifact.

\begin{table}[t]
\centering
\small
\caption{Latency and memory utilization under matched hardware and generation settings. Values are means over repeated generation runs; the released artifact contains per-run logs.}
\label{tab:latency}
\resizebox{\columnwidth}{!}{%
\begin{tabular}{lcccc}
\toprule
\textbf{Method} & \textbf{Prefill ms} & \textbf{Decode ms/token} & \textbf{Tokens/s} & \textbf{Peak memory} \\
\midrule
Standard RAG & 45 & 12.5 & 80.0 & 1.2GB \\
CAD & 45 & 24.1 & 41.5 & 2.4GB \\
AdaCAD & 45 & 24.5 & 40.8 & 2.4GB \\
Grounded Decoding & 48 & 25.0 & 40.0 & 2.5GB \\
\bottomrule
\end{tabular}%
}
\end{table}
\paragraph{Efficiency Analysis}
\label{sec:efficiency}
The computational overhead associated with the additional decoding stream is reported in \cref{tab:latency}. Grounded Decoding approximately doubles per-token decoding time relative to standard RAG, which is expected given that the method computes both full-context and retrieval-conditioned next-token distributions at each decoding step. The resulting latency remains comparable to that of contrastive decoding baselines that similarly rely on multiple decoding streams. The modest increase in memory consumption relative to CAD primarily arises from the retrieval-stream cache and the vocabulary-level fusion buffers. In deployment scenarios where inference overhead is a primary concern, the static variant provides a favorable trade-off, achieving most of the factuality and citation improvements while maintaining approximately the same per-token decoding cost as CAD.


\section{Conclusion}
\label{sec:discussion}
Grounded Decoding is most effective when the retrieved documents are relevant and the full RAG distribution diverges from the retrieval conditioned distribution at factual decoding positions. Unlike CAD, which amplifies tokens that become more likely with context, Grounded Decoding emphasizes tokens that remain probable under a document conditioned continuation. This makes decoding more directly sensitive to evidence compatibility, while the full RAG stream preserves query intent when the retrieved evidence is ambiguous. The KL barycenter combines these signals by retaining query conditioned behavior and adding retrieval grounded constraints through an auxiliary distribution. Additional analysis and extended discussion are provided in Appendix~\ref{app:discussion}.

We formulate faithful RAG decoding as a KL barycenter between the full RAG distribution and a per token retrieval only reference distribution. The objective has a unique closed form solution, reduces exactly to standard RAG when grounding is disabled, and moves monotonically toward evidence anchoring as grounding increases. The adaptive schedule strengthens retrieval influence when distributional conflict and retrieval trust are high. Under matched evaluation against standard RAG, CAD, AdaCAD, COIECD, CoCoA, DoLa, and kNN LM, Grounded Decoding achieves the strongest factuality and citation results among the evaluated decoding time methods while maintaining comparable fluency. Overall, the results support retrieval anchored probability fusion as a practical and transparent alternative to context aware logit subtraction.

\section*{Limitations}
\label{sec:limitations}

 Grounded Decoding introduces additional inference cost by requiring a second decoding stream; although this cost is comparable to other dual-stream contrastive decoding methods, it remains a non-trivial overhead relative to standard RAG. The method is dependent on retrieval quality and may faithfully reproduce incorrect or irrelevant retrieved content, and therefore does not replace improvements in retriever training or corpus curation. The adaptive schedule relies on validation-calibrated constants; despite equalizing tuning budgets across baselines, residual hyperparameter sensitivity may persist in deployment. The formulation operates as a local next-token fusion mechanism and does not guarantee global factual consistency across complete responses. In addition, the retrieval-only reference may become ambiguous when documents describe multiple entities or when reasoning requires aggregation across passages, limiting its effectiveness in such settings.


\bibliography{references}

\begin{thebibliography}{28}
\providecommand{\natexlab}[1]{#1}

\bibitem[{Asai et~al.(2024)Asai, Wu, Wang, Sil, and Hajishirzi}]{asai2024selfrag}
Akari Asai, Zeqiu Wu, Yizhong Wang, Avirup Sil, and Hannaneh Hajishirzi. 2024.
\newblock \href {https://arxiv.org/abs/2310.11511} {Self-{RAG}: Learning to retrieve, generate, and critique through self-reflection}.
\newblock In \emph{International Conference on Learning Representations}.

\bibitem[{Borgeaud et~al.(2022)Borgeaud, Mensch, Hoffmann, Cai, Rutherford, Millican, van~den Driessche, Lespiau, Damoc, Clark, de~Las~Casas, Guy, Menick, Ring, Hennigan, Huang, Maggiore, Jones, Cassirer, Brock, Paganini, Irving, Vinyals, Osindero, Simonyan, Rae, Elsen, and Sifre}]{borgeaud2022improving}
Sebastian Borgeaud, Arthur Mensch, Jordan Hoffmann, Trevor Cai, Eliza Rutherford, Katie Millican, George van~den Driessche, Jean-Baptiste Lespiau, Bogdan Damoc, Aidan Clark, Diego de~Las~Casas, Aurelia Guy, Jacob Menick, Roman Ring, Tom Hennigan, Saffron Huang, Loren Maggiore, Chris Jones, Albin Cassirer, and 9 others. 2022.
\newblock \href {https://arxiv.org/abs/2112.04426} {Improving language models by retrieving from trillions of tokens}.
\newblock In \emph{Proceedings of the 39th International Conference on Machine Learning}, pages 2206--2240. PMLR.

\bibitem[{Chuang et~al.(2024)Chuang, Xie, Luo, Kim, Glass, and He}]{chuang2024dola}
Yung-Sung Chuang, Yujia Xie, Hongyin Luo, Yoon Kim, James Glass, and Pengcheng He. 2024.
\newblock {DoLa}: Decoding by contrasting layers improves factuality in large language models.
\newblock In \emph{International Conference on Learning Representations}.

\bibitem[{Dathathri et~al.(2020)Dathathri, Madotto, Lan, Hung, Frank, Molino, Yosinski, and Liu}]{dathathri2020plug}
Sumanth Dathathri, Andrea Madotto, Janice Lan, Jane Hung, Eric Frank, Piero Molino, Jason Yosinski, and Rosanne Liu. 2020.
\newblock Plug and play language models: A simple approach to controlled text generation.
\newblock In \emph{International Conference on Learning Representations}.

\bibitem[{Fan et~al.(2018)Fan, Lewis, and Dauphin}]{fan2018hierarchical}
Angela Fan, Mike Lewis, and Yann Dauphin. 2018.
\newblock \href {https://doi.org/10.18653/v1/P18-1082} {Hierarchical neural story generation}.
\newblock In \emph{Proceedings of the 56th Annual Meeting of the Association for Computational Linguistics (Volume 1: Long Papers)}, pages 889--898. Association for Computational Linguistics.

\bibitem[{Gao et~al.(2023)Gao, Yen, Yu, and Chen}]{gao2023enabling}
Tianyu Gao, Howard Yen, Jiatong Yu, and Danqi Chen. 2023.
\newblock \href {https://doi.org/10.18653/v1/2023.emnlp-main.398} {Enabling large language models to generate text with citations}.
\newblock In \emph{Proceedings of the 2023 Conference on Empirical Methods in Natural Language Processing}, pages 6465--6488. Association for Computational Linguistics.

\bibitem[{Guu et~al.(2020)Guu, Lee, Tung, Pasupat, and Chang}]{guu2020realm}
Kelvin Guu, Kenton Lee, Zora Tung, Panupong Pasupat, and Ming-Wei Chang. 2020.
\newblock {REALM}: Retrieval-augmented language model pre-training.
\newblock In \emph{Proceedings of the 37th International Conference on Machine Learning}, pages 3929--3938. PMLR.

\bibitem[{Holtzman et~al.(2020)Holtzman, Buys, Du, Forbes, and Choi}]{holtzman2020curious}
Ari Holtzman, Jan Buys, Li~Du, Maxwell Forbes, and Yejin Choi. 2020.
\newblock The curious case of neural text degeneration.
\newblock In \emph{International Conference on Learning Representations}.

\bibitem[{Huang et~al.(2025)Huang, Yu, Ma, Zhong, Feng, Wang, Chen, Peng, Feng, Qin, and Liu}]{huang2025survey}
Lei Huang, Weijiang Yu, Weitao Ma, Weihong Zhong, Zhangyin Feng, Haotian Wang, Qianglong Chen, Weihua Peng, Xiaocheng Feng, Bing Qin, and Ting Liu. 2025.
\newblock \href {https://doi.org/10.1145/3703155} {A survey on hallucination in large language models: Principles, taxonomy, challenges, and open questions}.
\newblock \emph{ACM Transactions on Information Systems}, 43(2):1--55.
\newblock Also available as arXiv:2311.05232.

\bibitem[{Izacard et~al.(2022)Izacard, Caron, Hosseini, Riedel, Bojanowski, Joulin, and Grave}]{izacard2022unsupervised}
Gautier Izacard, Mathilde Caron, Lucas Hosseini, Sebastian Riedel, Piotr Bojanowski, Armand Joulin, and Edouard Grave. 2022.
\newblock Unsupervised dense information retrieval with contrastive learning.
\newblock \emph{Transactions on Machine Learning Research}.

\bibitem[{Ji et~al.(2026)Ji, Tutunov, Zimmer, and Bou-Ammar}]{ji2026decoding}
Xiaotong Ji, Rasul Tutunov, Matthieu Zimmer, and Haitham Bou-Ammar. 2026.
\newblock Decoding as optimisation on the probability simplex: From top-{K} to top-{P} (nucleus) to best-of-{K} samplers.
\newblock \emph{arXiv preprint arXiv:2602.18292}.

\bibitem[{Ji et~al.(2023)Ji, Lee, Frieske, Yu, Su, Xu, Ishii, Bang, Madotto, and Fung}]{ji2023survey}
Ziwei Ji, Nayeon Lee, Rita Frieske, Tiezheng Yu, Dan Su, Yan Xu, Etsuko Ishii, Ye~Jin Bang, Andrea Madotto, and Pascale Fung. 2023.
\newblock \href {https://doi.org/10.1145/3571730} {Survey of hallucination in natural language generation}.
\newblock \emph{ACM Computing Surveys}, 55(12):1--38.

\bibitem[{Khandelwal et~al.(2025)Khandelwal, Gupta, and Agrawal}]{khandelwal2025cocoa}
Anant Khandelwal, Manish Gupta, and Puneet Agrawal. 2025.
\newblock {CoCoA}: Confidence- and context-aware adaptive decoding for resolving knowledge conflicts in large language models.
\newblock In \emph{Proceedings of the 2025 Conference on Empirical Methods in Natural Language Processing}, pages 6835--6855. Association for Computational Linguistics.

\bibitem[{Khandelwal et~al.(2020)Khandelwal, Levy, Jurafsky, Zettlemoyer, and Lewis}]{khandelwal2020generalization}
Urvashi Khandelwal, Omer Levy, Dan Jurafsky, Luke Zettlemoyer, and Mike Lewis. 2020.
\newblock Generalization through memorization: Nearest neighbor language models.
\newblock In \emph{International Conference on Learning Representations}.

\bibitem[{Krause et~al.(2021)Krause, Gotmare, McCann, Keskar, Joty, Socher, and Rajani}]{krause2021gedi}
Ben Krause, Akhilesh~Deepak Gotmare, Bryan McCann, Nitish~Shirish Keskar, Shafiq Joty, Richard Socher, and Nazneen~Fatema Rajani. 2021.
\newblock \href {https://doi.org/10.18653/v1/2021.findings-emnlp.424} {{GeDi}: Generative discriminator guided sequence generation}.
\newblock In \emph{Findings of the Association for Computational Linguistics: EMNLP 2021}, pages 4929--4952. Association for Computational Linguistics.

\bibitem[{Kwiatkowski et~al.(2019)Kwiatkowski, Palomaki, Redfield, Collins, Parikh, Alberti, Epstein, Polosukhin, Devlin, Lee, Toutanova, Jones, Kelcey, Chang, Dai, Uszkoreit, Le, and Petrov}]{kwiatkowski2019natural}
Tom Kwiatkowski, Jennimaria Palomaki, Olivia Redfield, Michael Collins, Ankur Parikh, Chris Alberti, Danielle Epstein, Illia Polosukhin, Jacob Devlin, Kenton Lee, Kristina Toutanova, Llion Jones, Matthew Kelcey, Ming-Wei Chang, Andrew~M. Dai, Jakob Uszkoreit, Quoc Le, and Slav Petrov. 2019.
\newblock \href {https://doi.org/10.1162/tacl_a_00276} {Natural questions: A benchmark for question answering research}.
\newblock \emph{Transactions of the Association for Computational Linguistics}, 7:452--466.

\bibitem[{Lewis et~al.(2020)Lewis, Perez, Piktus, Petroni, Karpukhin, Goyal, K{\"u}ttler, Lewis, Yih, Rockt{\"a}schel, Riedel, and Kiela}]{lewis2020retrieval}
Patrick Lewis, Ethan Perez, Aleksandra Piktus, Fabio Petroni, Vladimir Karpukhin, Naman Goyal, Heinrich K{\"u}ttler, Mike Lewis, Wen-tau Yih, Tim Rockt{\"a}schel, Sebastian Riedel, and Douwe Kiela. 2020.
\newblock Retrieval-augmented generation for knowledge-intensive {NLP} tasks.
\newblock In \emph{Advances in Neural Information Processing Systems}, volume~33, pages 9459--9474.

\bibitem[{Li et~al.(2023)Li, Holtzman, Fried, Liang, Eisner, Hashimoto, Zettlemoyer, and Lewis}]{li2023contrastive}
Xiang~Lisa Li, Ari Holtzman, Daniel Fried, Percy Liang, Jason Eisner, Tatsunori~B. Hashimoto, Luke Zettlemoyer, and Mike Lewis. 2023.
\newblock \href {https://doi.org/10.18653/v1/2023.acl-long.687} {Contrastive decoding: Open-ended text generation as optimization}.
\newblock In \emph{Proceedings of the 61st Annual Meeting of the Association for Computational Linguistics (Volume 1: Long Papers)}, pages 12286--12312. Association for Computational Linguistics.

\bibitem[{Longpre et~al.(2021)Longpre, Perisetla, Chen, Ramesh, DuBois, and Singh}]{longpre2021entity}
Shayne Longpre, Kartik Perisetla, Anthony Chen, Nikhil Ramesh, Chris DuBois, and Sameer Singh. 2021.
\newblock \href {https://doi.org/10.18653/v1/2021.emnlp-main.565} {Entity-based knowledge conflicts in question answering}.
\newblock In \emph{Proceedings of the 2021 Conference on Empirical Methods in Natural Language Processing}, pages 7052--7063. Association for Computational Linguistics.

\bibitem[{{Meta AI}(2024)}]{meta2024llama3}
{Meta AI}. 2024.
\newblock Meta {Llama} 3 model card.
\newblock \url{https://github.com/meta-llama/llama3/blob/main/MODEL_CARD.md}.
\newblock Model card, accessed 2026-05-15.

\bibitem[{Min et~al.(2023)Min, Krishna, Lyu, Lewis, Yih, Koh, Iyyer, Zettlemoyer, and Hajishirzi}]{min2023factscore}
Sewon Min, Kalpesh Krishna, Xinxi Lyu, Mike Lewis, Wen-tau Yih, Pang~Wei Koh, Mohit Iyyer, Luke Zettlemoyer, and Hannaneh Hajishirzi. 2023.
\newblock \href {https://doi.org/10.18653/v1/2023.emnlp-main.741} {{FActScore}: Fine-grained atomic evaluation of factual precision in long form text generation}.
\newblock In \emph{Proceedings of the 2023 Conference on Empirical Methods in Natural Language Processing}, pages 12076--12100. Association for Computational Linguistics.

\bibitem[{Nguyen et~al.(2025)Nguyen, Baker, Neo, Roush, Kirsch, and Shwartz-Ziv}]{nguyen2025minp}
Minh Nguyen, Andrew Baker, Clement Neo, Allen Roush, Andreas Kirsch, and Ravid Shwartz-Ziv. 2025.
\newblock Turning up the heat: Min-{P} sampling for creative and coherent {LLM} outputs.
\newblock In \emph{International Conference on Learning Representations}.
\newblock Oral; arXiv:2407.01082.

\bibitem[{Ouyang et~al.(2022)Ouyang, Wu, Jiang, Almeida, Wainwright, Mishkin, Zhang, Agarwal, Slama, Ray, Schulman, Hilton, Kelton, Miller, Simens, Askell, Welinder, Christiano, Leike, and Lowe}]{ouyang2022training}
Long Ouyang, Jeffrey Wu, Xu~Jiang, Diogo Almeida, Carroll~L. Wainwright, Pamela Mishkin, Chong Zhang, Sandhini Agarwal, Katarina Slama, Alex Ray, John Schulman, Jacob Hilton, Fraser Kelton, Luke Miller, Maddie Simens, Amanda Askell, Peter Welinder, Paul~F. Christiano, Jan Leike, and Ryan Lowe. 2022.
\newblock Training language models to follow instructions with human feedback.
\newblock In \emph{Advances in Neural Information Processing Systems}, volume~35, pages 27730--27744.

\bibitem[{Rafailov et~al.(2023)Rafailov, Sharma, Mitchell, Manning, Ermon, and Finn}]{rafailov2023direct}
Rafael Rafailov, Archit Sharma, Eric Mitchell, Christopher~D. Manning, Stefano Ermon, and Chelsea Finn. 2023.
\newblock Direct preference optimization: Your language model is secretly a reward model.
\newblock In \emph{Advances in Neural Information Processing Systems}, volume~36.

\bibitem[{Shi et~al.(2024)Shi, Han, Lewis, Tsvetkov, Zettlemoyer, and Yih}]{shi2024trusting}
Weijia Shi, Xiaochuang Han, Mike Lewis, Yulia Tsvetkov, Luke Zettlemoyer, and Wen-tau Yih. 2024.
\newblock \href {https://doi.org/10.18653/v1/2024.naacl-short.69} {Trusting your evidence: Hallucinate less with context-aware decoding}.
\newblock In \emph{Proceedings of the 2024 Conference of the North American Chapter of the Association for Computational Linguistics: Human Language Technologies (Volume 2: Short Papers)}, pages 783--791. Association for Computational Linguistics.

\bibitem[{Wang et~al.(2025)Wang, Prasad, Stengel-Eskin, and Bansal}]{wang2025adacad}
Han Wang, Archiki Prasad, Elias Stengel-Eskin, and Mohit Bansal. 2025.
\newblock {AdaCAD}: Adaptively decoding to balance conflicts between contextual and parametric knowledge.
\newblock In \emph{Proceedings of the 2025 Conference of the Nations of the Americas Chapter of the Association for Computational Linguistics: Human Language Technologies (Volume 1: Long Papers)}, pages 11636--11652. Association for Computational Linguistics.

\bibitem[{Yang and Klein(2021)}]{yang2021fudge}
Kevin Yang and Dan Klein. 2021.
\newblock \href {https://doi.org/10.18653/v1/2021.naacl-main.276} {{FUDGE}: Controlled text generation with future discriminators}.
\newblock In \emph{Proceedings of the 2021 Conference of the North American Chapter of the Association for Computational Linguistics: Human Language Technologies}, pages 3511--3535. Association for Computational Linguistics.

\bibitem[{Yuan et~al.(2024)Yuan, Yang, Wang, Liu, Zhao, and Liu}]{yuan2024coiecd}
Xiaowei Yuan, Zhao Yang, Yequan Wang, Shengping Liu, Jun Zhao, and Kang Liu. 2024.
\newblock \href {https://doi.org/10.18653/v1/2024.findings-acl.234} {Discerning and resolving knowledge conflicts through adaptive decoding with contextual information-entropy constraint}.
\newblock In \emph{Findings of the Association for Computational Linguistics: ACL 2024}, pages 3903--3922. Association for Computational Linguistics.

\end{thebibliography}

\appendix

\section{Extended Related Work}
\label{app:related_work}
\paragraph{Grounded Decoding}
Grounded Decoding builds on this line of work but differs in formulation and objective design. Instead of contrasting against a context-free distribution, it introduces a retrieval-only distribution conditioned solely on retrieved documents. This formulation directly promotes tokens supported by external evidence rather than suppressing generic continuations. Moreover, the decoding objective is derived from an explicit optimization problem over the probability simplex, yielding a principled probabilistic interpretation and well-behaved boundary properties.
\paragraph{Decoding-Time Interventions}
Decoding-time interventions have been widely studied as lightweight alternatives to retraining large language models. Contrastive Decoding compares an expert model with an amateur or degraded model to suppress generic generations and improve specificity \citep{li2023contrastive}. DoLa leverages differences between early and late transformer layers to enhance factuality in generation \citep{chuang2024dola}. kNN-LM interpolates parametric predictions with a retrieval datastore to improve fluency and factual consistency \citep{khandelwal2020generalization}.

\paragraph{Controllable Generation}
Other approaches focus on controllable generation using external signals. FUDGE conditions generation on future attribute classifiers without retraining the base model \citep{yang2021fudge}. PPLM steers generation using gradient-based updates from attribute models \citep{dathathri2020plug}. GeDi uses discriminative models to guide decoding toward desired attributes \citep{krause2021gedi}.

\paragraph{Grounding in Retrieval}
While these methods improve controllability or fluency, they are not explicitly designed to enforce token-level grounding in retrieved evidence. In contrast, Grounded Decoding directly integrates retrieval information into the decoding objective to improve evidence alignment at every generation step.
\section{Additional Proof Details}
\label{app:proof_details}

The Hessian of \((1+\rho_t)\sum_i q_i\log q_i\) on the positive orthant is diagonal with entries \((1+\rho_t)/q_i\), which are strictly positive for \(q_i>0\). The remaining terms in \cref{eq:main_objective} are linear in \(q\). The objective is therefore strictly convex on the relative interior of the simplex, and the optimum is unique. Since \({\pfull}_t\) and \({\pret}_t\) have full support in exact softmax arithmetic, the optimizer also has full support, and the first-order condition used in the proof of \cref{thm:closed_form} applies coordinate-wise.

Numerically, both reference distributions can have entries that underflow in low-precision arithmetic. We compute the fusion entirely in log space, $\log {\qgd}_t = (1-\eta_t)\log{\pfull}_t + \eta_t \log{\pret}_t$, followed by a single $\logsumexp$ for the normalizer. When the base model runs in fp16 or bfloat16, the fusion and divergence computations are cast to float32 for stability. This is purely a numerical convenience and does not change the closed-form solution.

\Cref{eq:closed_form} can be viewed as a two-expert product model with exponents that sum to one, in the sense of a logarithmic opinion pool. The interpretation is useful but should not be overstated, because both experts are the same base language model under different prompts, and their errors are typically correlated. The method is therefore not an ensemble of independently trained models, and the gains we observe cannot be attributed to error decorrelation as in traditional ensembling.

An arithmetic mixture \((1-\alpha){\pfull}_t+\alpha{\pret}_t\) also gives exact fallback to the full distribution at \(\alpha=0\), but it behaves differently in the tails. The geometric mixture in \cref{eq:closed_form} penalizes tokens that either distribution assigns very low probability to, while an arithmetic mixture can keep such tokens alive when one distribution assigns them high mass. Which behavior is preferable is an empirical question; the ablation in \cref{tab:ablations} shows that the geometric form outperforms the arithmetic form on factuality and citation metrics in our setting, though the gap is small enough that an arithmetic mixture remains a reasonable choice when a logarithmic implementation is inconvenient.

\section{Why Not a Linear Logit Objective}
\label{app:linear_logit}

A natural but flawed alternative to \cref{eq:main_objective} is to minimize a linear logit term plus a KL penalty to an interpolated target distribution. That objective can produce a closed form, but it has two problems in the RAG setting. First, if the target distribution already includes \({\pfull}_t\), multiplying again by a softmax-shaped term double-counts the full-context model and produces a rule that no longer reduces to standard RAG in any limit. Second, the direction of the regularization parameter becomes easy to misinterpret: in such an objective, larger values strengthen the KL anchor toward the interpolated target, whereas smaller values collapse the rule toward greedy logit maximization, so the parameter's qualitative effect inverts as the limit is taken. The KL-barycenter formulation avoids both issues, because it interpolates between two distributions directly and makes the grounding weight monotone by construction. The same monotonicity is essential for the adaptive schedule of \cref{sec:adaptive} to behave as intended: increasing \(\rho_t\) in response to higher conflict or higher retrieval trust always shifts mass toward the retrieval reference, never away from it.

\section{Baseline Configuration}
\label{app:baseline_config}

All baselines are evaluated under a matched decoding protocol. Hyperparameters are tuned exclusively on the validation set, with a fixed and comparable tuning budget across methods, as summarized in \cref{tab:baseline_config} 
\begin{table*}[t]
\centering
\small
\caption{Baseline configuration and tuning budget. All hyperparameters are selected on validation data only.}
\label{tab:baseline_config}
\resizebox{\textwidth}{!}{%
\begin{tabular}{lll}
\toprule
\textbf{Method} & \textbf{Tuned parameters} & \textbf{Implementation notes} \\
\midrule
Standard RAG & temperature, top-$p$, maximum length & Same retrieved passages and prompt template as all methods. \\
CAD & $\beta\in\{0.1,0.3,0.5,1.0,1.5,2.0\}$ & Full-context stream contrasted with query-only/no-context stream. \\
AdaCAD & published adaptation strength grid & Uses the authors' conflict-adaptive rule under the same retrieval outputs. \\
COIECD & published entropy/conflict grid & Uses contextual information entropy adaptation under the same decoding budget. \\
CoCoA & published confidence/divergence grid & Uses confidence and context-awareness adaptation under matched prompts. \\
DoLa & layer pair/grid from validation & Same generation length and truncation rule as RAG. \\
kNN-LM & $k$, interpolation weight, datastore layer & Datastore built from retrieved passages only for each example. \\
Grounded static & $\rho$ & Constant grounding weight selected on validation. \\
Grounded adaptive & $\rho_{\max},\tau_r,\gamma$ & Conflict and retrieval-trust schedule selected on validation. \\
\bottomrule
\end{tabular}%
}
\end{table*}

\section{Implementation Details and Algorithm}
\label{app:implementation_details}

We now describe the implementation details and provide the formal algorithmic procedure for Grounded Decoding.The decoder maintains two parallel streams that share identical model parameters, tokenizer, and decoding hyperparameters. The full stream is conditioned on a prompt containing the system message, query \(x\), and retrieved documents \(\docs\), whereas the retrieval stream is conditioned on a matched prompt that excludes the query and includes only the system message and documents.

At each decoding step \(t\), both streams generate next-token distributions conditioned on the same generated prefix \(y_{<t}\). The fused distribution is computed as the KL-barycenter solution defined in \cref{eq:closed_form}, using the adaptive weight specified in \cref{eq:rho_t}. Standard decoding transformations, including temperature scaling and truncation strategies such as top-\(p\), top-\(k\), or min-\(p\), are applied to the fused distribution \({\qgd}_t\) prior to sampling. Following token generation, both streams update their key-value caches with the sampled token \(y_t\), and decoding proceeds autoregressively.

The proposed method operates entirely at the probability level, following the final vocabulary projection. It does not modify internal model components such as attention layers, residual connections, MLP blocks, normalization modules, positional encodings (e.g., RoPE or ALiBi), or parallelization schemes including tensor or pipeline sharding. While these architectural components influence the individual stream distributions, the proposed fusion mechanism is applied as a post-hoc operation on logits.

From an implementation perspective, computations are performed in log-space using \texttt{log\_softmax}. When the underlying model operates in fp16 or bfloat16 precision, the fusion and divergence computations are cast to float32 to mitigate numerical underflow in low-probability regions of the vocabulary distribution. All decoding constraints, including top-\(p\), top-\(k\), min-\(p\), and temperature scaling, are applied after fusion and are kept consistent across all baselines unless explicitly stated otherwise \citep{holtzman2020curious,fan2018hierarchical,nguyen2025minp}.

The two-stream design introduces additional inference overhead. While the document prefix for the retrieval stream can be cached and reused across time steps, each generated token requires one forward pass for the full stream, one forward pass for the retrieval stream, and a subsequent fusion operation over the vocabulary. Overall, the computational cost is comparable to other contrastive decoding approaches that maintain multiple streams, with minor variations arising from prompt-length differences.We empirically evaluate this computational overhead in \cref{sec:results} under controlled settings with matched context length, batch size, numerical precision, and hardware configuration.
\begin{algorithm}[t]
\caption{Grounded Decoding with a per-token retrieval reference.}
\label{alg:grounded_decoding}
\begin{algorithmic}[1]
\REQUIRE model $M$, query $x$, retrieved documents $\docs$, retriever scores $s_1,\ldots,s_m$, maximum grounding weight $\rho_{\max}$, generation limit $T$
\STATE Format a full prompt with the system instruction, $x$, and $\docs$.
\STATE Format a retrieval-only prompt with the same template and system instruction, containing $\docs$ and omitting only $x$.
\STATE Initialize full-stream and retrieval-stream KV caches with their respective prompts.
\STATE Compute retrieval trust $r(x,\docs)$ from \cref{eq:trust_score}.
\FOR{$t=1$ to $T$}
    \STATE ${\pfull}_t \gets p_\theta(\cdot\mid x,\docs,y_{<t})$ from the full stream.
    \STATE ${\pret}_t \gets p_\theta(\cdot\mid \docs,y_{<t})$ from the retrieval stream.
    \STATE $c_t \gets \JS({\pfull}_t,{\pret}_t)/\log 2$.
    \STATE $\rho_t \gets \rho_{\max}\,r(x,\docs)\,c_t$; $\eta_t \gets \rho_t/(1+\rho_t)$.
    \STATE $\log q_t \gets (1-\eta_t)\log{\pfull}_t+\eta_t\log{\pret}_t-\log Z_t$.
    \STATE Apply the matched truncation rule and sample or select $y_t$.
    \STATE Append $y_t$ to both streams and update both KV caches.
\ENDFOR
\RETURN $y_1,\ldots,y_T$
\end{algorithmic}
\end{algorithm}

\section{Extended Discussion}
\label{app:discussion}
The ablation study in \cref{tab:ablations} systematically isolates the contribution of each component of the proposed method. The strong performance of the arithmetic interpolation baseline indicates that the retrieval-conditioned reference itself accounts for a substantial portion of the improvement over standard RAG. The geometric KL barycenter nevertheless provides additional and consistent gains, suggesting that the multiplicative fusion more effectively suppresses unsupported continuations favored by only one stream. Token-level analysis supports this interpretation, showing that adaptive grounding primarily activates on content-bearing factual tokens rather than uniformly across the generated response.

Several limitations follow from this formulation. The proposed method operates as a decoding-time intervention on next-token probabilities and should not be interpreted as evidence of causal or mechanistic factual reasoning within the model. Because the method explicitly increases reliance on retrieved evidence, retrieval quality and corpus reliability become more critical than in conventional RAG systems, increasing sensitivity to retrieval noise or poisoned corpora. In addition, the grounding mechanism operates locally at the token level and does not explicitly enforce global consistency across a complete response. Applications requiring stronger factual guarantees may therefore benefit from combining Grounded Decoding with downstream verification or consistency-checking procedures.

\section{Prompt Templates}
\label{app:prompts}

We use matched prompt templates throughout. The full stream contains the system message, the query, the retrieved documents, and the assistant prefix; the retrieval stream uses the same system message and document formatting but removes the query.

\begin{quote}\scriptsize\ttfamily
\begin{tabular}{@{}l@{}}
System: Use only provided documents.\\
Cite documents when needed.\\[2pt]
User: Question: \{question\}\\
Documents:\\
Doc 1: \{passage\_1\}\\
...\\
Doc m: \{passage\_m\}\\[2pt]
Assistant:
\end{tabular}
\end{quote}

The matched retrieval stream is:

\begin{quote}\scriptsize\ttfamily
\begin{tabular}{@{}l@{}}
System: Use only provided documents.\\
Cite documents when needed.\\[2pt]
User: Documents:\\
Doc 1: \{passage\_1\}\\
...\\
Doc m: \{passage\_m\}\\[2pt]
Assistant:
\end{tabular}
\end{quote}

\section{Ethics and Broader Impact}
\label{app:ethics}

Grounded Decoding is designed to reduce unsupported generation by increasing the influence of retrieved evidence on the decoding process. While this can improve reliability in settings where the retrieval corpus is curated and up to date, it may also amplify harmful, biased, or incorrect content when the underlying corpus is polluted, untrusted, or adversarially manipulated. Consequently, deployment should incorporate retrieval-quality filtering, source verification, and appropriate safeguards on retrieved content. In addition, systems based on this approach may benefit from user-facing uncertainty indicators and transparency mechanisms to support responsible use in downstream applications.

\section{Supplementary Figures}
\label{app:supplementary_figures}

We include supplementary figures to provide further intuition, implementation detail, and empirical analysis for Grounded Decoding. These figures illustrate the adaptive grounding weight, the full and retrieval-only prompt streams, the geometry of KL-based fusion, factuality and citation results, ablation trends, and token-level grounding behavior.

\begin{figure*}[t]
\centering
\includegraphics[width=0.6\textwidth]{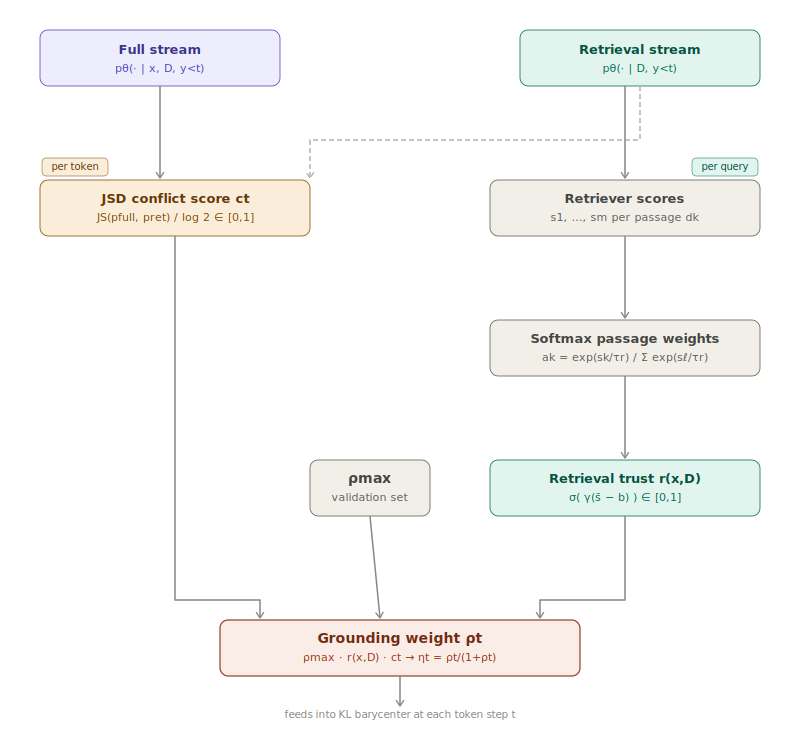}
\caption{Computation of the adaptive grounding weight $\rho_t$ (\ref{sec:adaptive}). The weight is defined as the product of three components: a constant $\rho_{\max}$ tuned on a validation split, a token-level conflict score $c_t$, and an instance-level retrieval trust score $r(x,\mathcal{D})$. The conflict score $c_t$ is computed as the normalized Jensen--Shannon divergence between the full RAG distribution and the retrieval-only distribution at each decoding step. The retrieval trust score $r(x,\mathcal{D})$ is derived from retriever passage scores via softmax normalization followed by a sigmoid gating function and is computed once per query. Low conflict or weak retrieval keeps $\rho_t$ near zero, resulting in behavior similar to standard RAG. High conflict combined with strong retrieval increases $\rho_t$ and shifts the KL barycenter toward the retrieval-only reference. The dashed arrow indicates that the retrieval stream contributes to both branches of the computation.}
\label{fig:adaptive_grounding}
\end{figure*}

\begin{figure*}[t]
\centering
\includegraphics[width=0.8\textwidth]{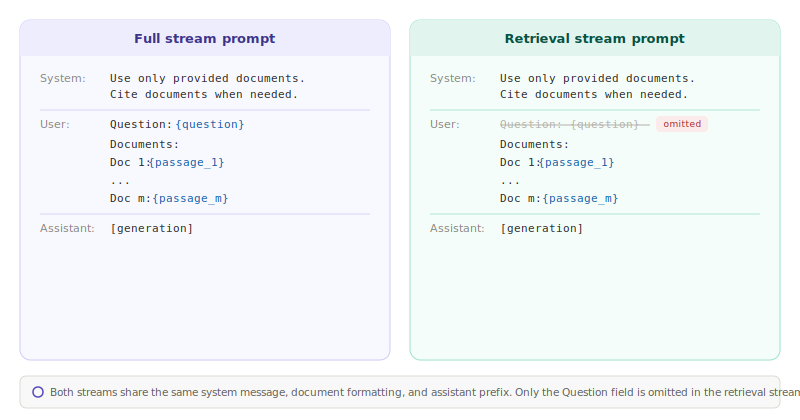}
\caption{Prompt templates used for the full stream and retrieval stream. Both streams share identical system message, document formatting, and assistant prefix. The retrieval stream omits only the \texttt{Question: \{question\}} field. This comparison demonstrates that the only distinction between the two streams is the presence of the query field.}
\label{fig:prompt_template_comparison}
\end{figure*}

\begin{figure*}
\centering
\includegraphics[width=0.95\textwidth]{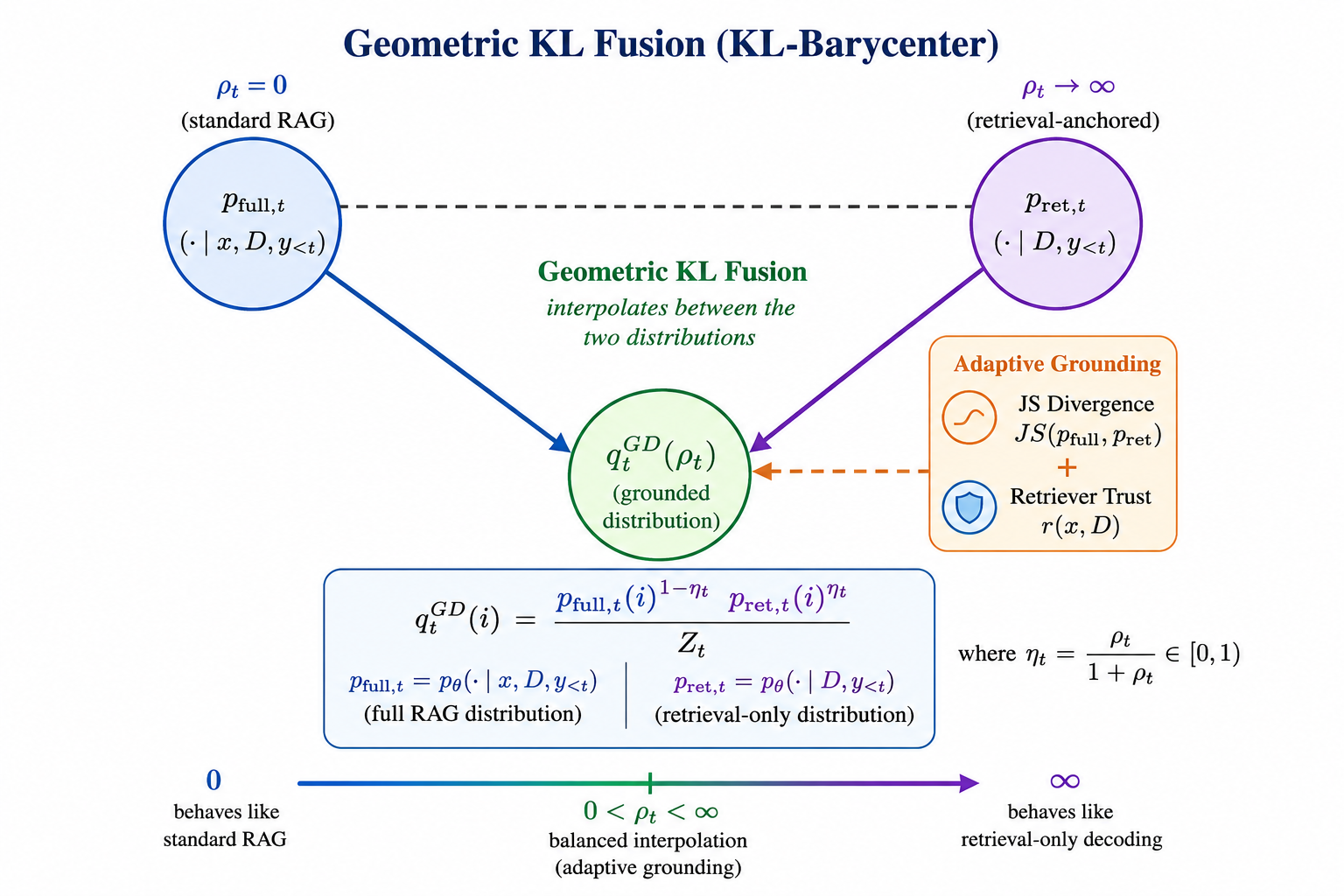}
\caption{Geometric interpretation of Grounded Decoding as KL-barycenter probability fusion between the full RAG distribution $p_{\text{full},t}$ and the retrieval-only reference distribution $p_{\text{ret},t}$. The grounding parameter $\rho_t$ controls interpolation from standard RAG behavior ($\rho_t = 0$) toward retrieval-anchored decoding ($\rho_t \rightarrow \infty$). The grounded distribution $q^{\text{GD}}_t$ is obtained through geometric KL fusion. Adaptive grounding further adjusts retrieval influence using distributional conflict and retriever trust signals.}
\label{fig:kl_fusion}
\end{figure*}

\begin{figure*}[t]
\centering
\includegraphics[width=0.95\textwidth]{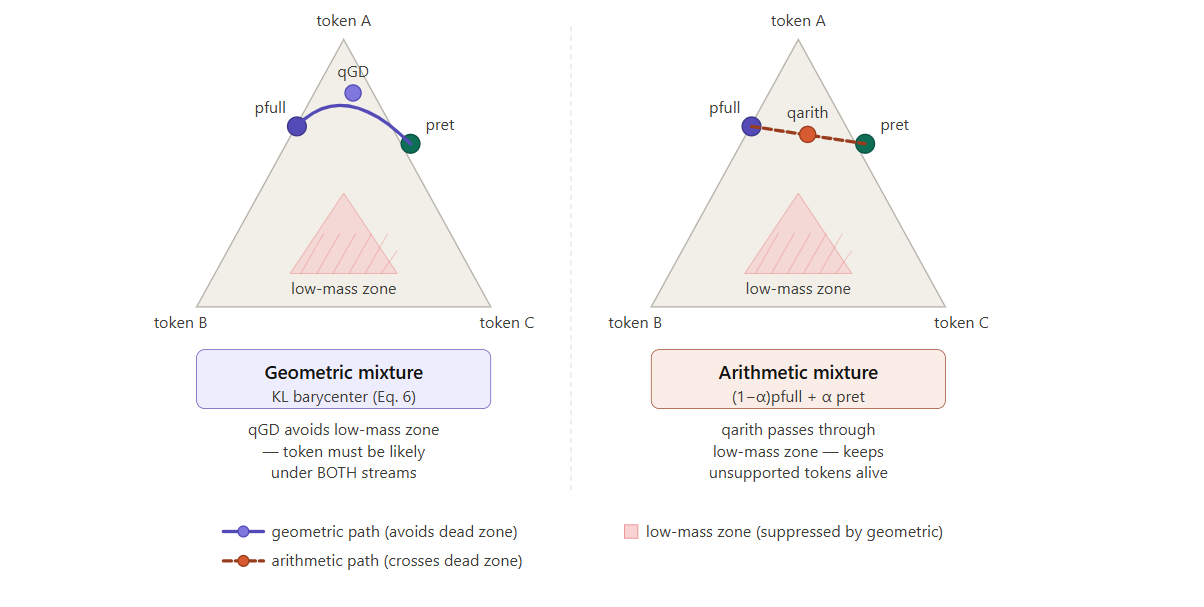}
\caption{Geometric interpretation of Grounded Decoding on the probability simplex. Each vertex corresponds to a deterministic next-token distribution. The geometric KL barycenter traces a curved path along the high-probability ridge between $p_{\text{full}}$ and $p_{\text{ret}}$, ensuring that selected tokens remain plausible under both distributions. In contrast, the arithmetic mixture follows a linear interpolation that passes through low-density regions of the simplex, allowing tokens supported by only one stream to persist. The hatched region denotes low-probability mass under both distributions, where the geometric formulation suppresses candidate tokens while the arithmetic mixture does not.}
\label{fig:grounded_simplex}
\end{figure*}
\begin{figure}[t]
\centering
\includegraphics[width=\linewidth]{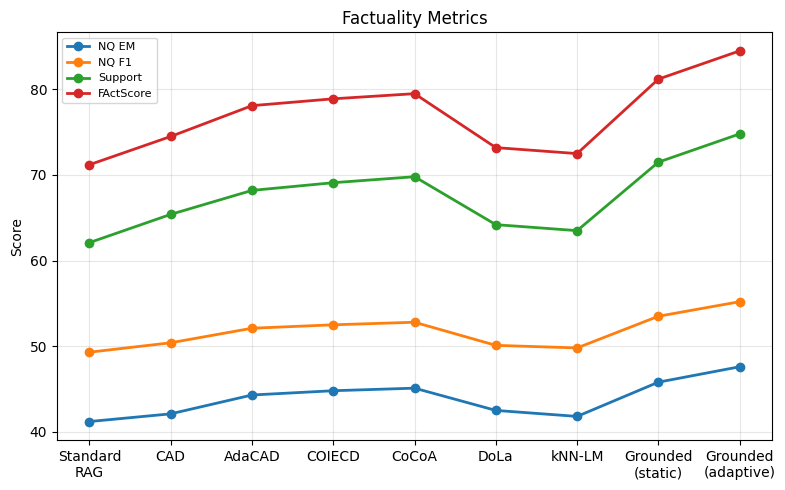}

\caption{Comparison of factuality and QA performance across different retrieval-augmented generation methods. We evaluate Standard RAG, CAD variants, CoCoA, DoLa, kNN-LM, and Grounded Decoding (static and adaptive) on Natural Questions (NQ EM and F1), Support verification accuracy, and FActScore. Grounded (adaptive) consistently achieves the best performance across all metrics, demonstrating improved factual consistency and answer quality over both retrieval-only and decoding-time baselines.}

\label{fig:factuality_metrics}
\end{figure}
\begin{figure}[t]
\centering
\includegraphics[width=\linewidth]{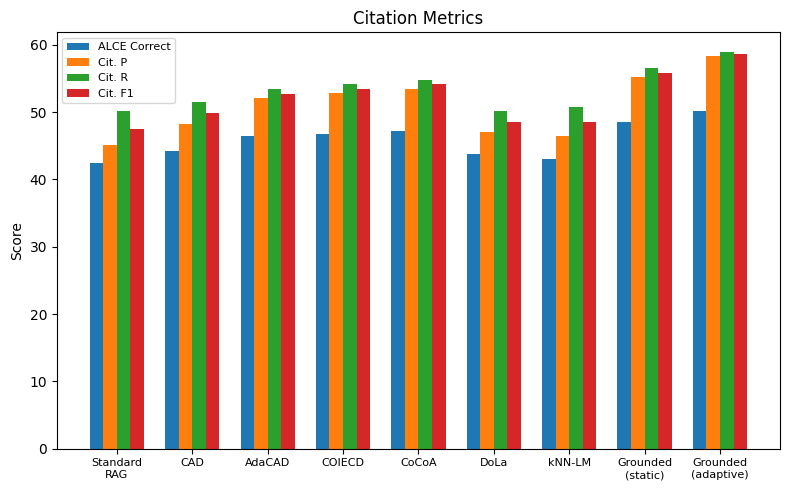}

\caption{Citation quality evaluation across different retrieval-augmented generation methods. We report ALCE correctness, citation precision (Cit. P), recall (Cit. R), and F1 score across Standard RAG, CAD variants, CoCoA, DoLa, kNN-LM, and Grounded Decoding (static and adaptive). Grounded (adaptive) achieves the highest performance across all citation metrics, indicating improved grounding quality, better attribution precision, and more reliable factual referencing compared to prior RAG and decoding-time baselines.}

\label{fig:citation_metrics}
\end{figure}

\begin{figure}[t]
\centering
\includegraphics[width=\linewidth]{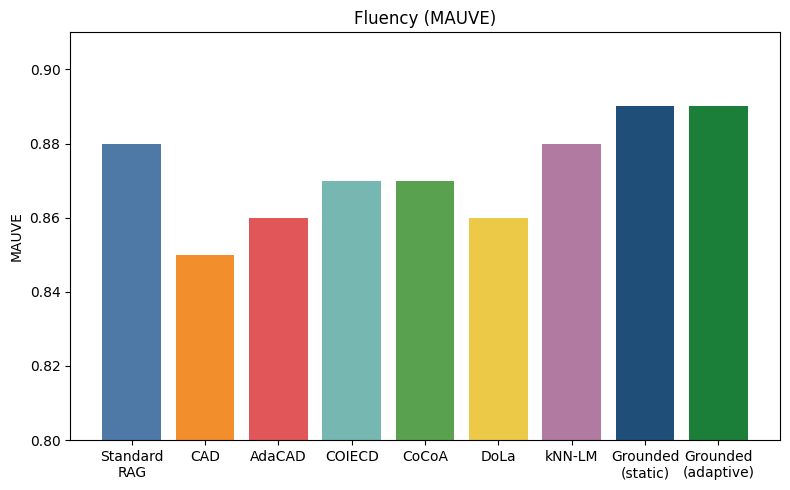}

\caption{Fluency evaluation across different retrieval-augmented generation methods measured using MAUVE scores. We compare Standard RAG, CAD variants, CoCoA, DoLa, kNN-LM, and Grounded Decoding (static and adaptive). Grounded (adaptive) achieves the highest MAUVE score, indicating improved fluency and naturalness while maintaining strong factual grounding compared to prior retrieval-augmented and decoding-time baselines.}

\label{fig:mauve_metrics}
\end{figure}

\begin{figure}[t]
\centering
\includegraphics[width=\linewidth]{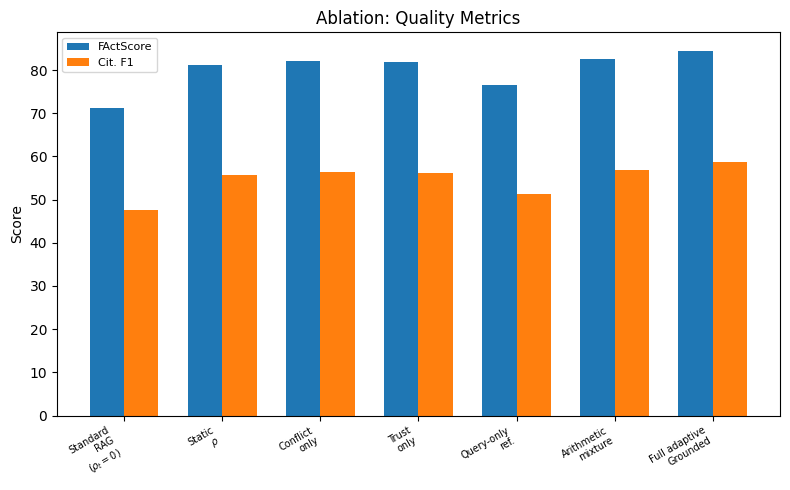}

\caption{Ablation study on quality metrics comparing different variants of the proposed method. We evaluate Standard RAG ($\rho_t=0$), Static $\rho$, Conflict-only, Trust-only, Query-only reference, arithmetic mixture, and the full adaptive Grounded Decoding model. We report FActScore and citation F1. The full adaptive model consistently achieves the best performance, demonstrating the effectiveness of jointly modeling trust and conflict signals in grounded decoding.}

\label{fig:ablation_quality_metrics}
\end{figure}

\begin{figure}[t]
\centering
\includegraphics[width=\linewidth]{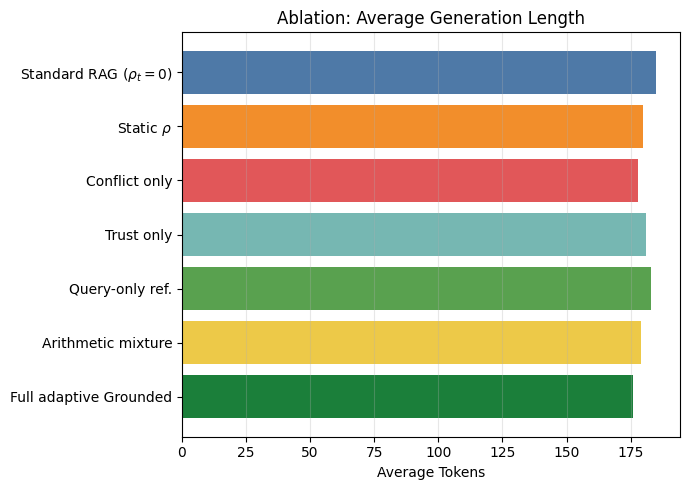}

\caption{Ablation study on generation length measured by average number of tokens across different variants of the proposed method. We evaluate Standard RAG ($\rho_t=0$), Static $\rho$, Conflict-only, Trust-only, Query-only reference, arithmetic mixture, and the full adaptive Grounded Decoding model. The full adaptive model generates more concise outputs compared to most ablations while maintaining strong performance, indicating improved efficiency without sacrificing grounding quality.}

\label{fig:ablation_avg_tokens}
\end{figure}

\begin{figure}[t]
\centering
\includegraphics[width=\linewidth]{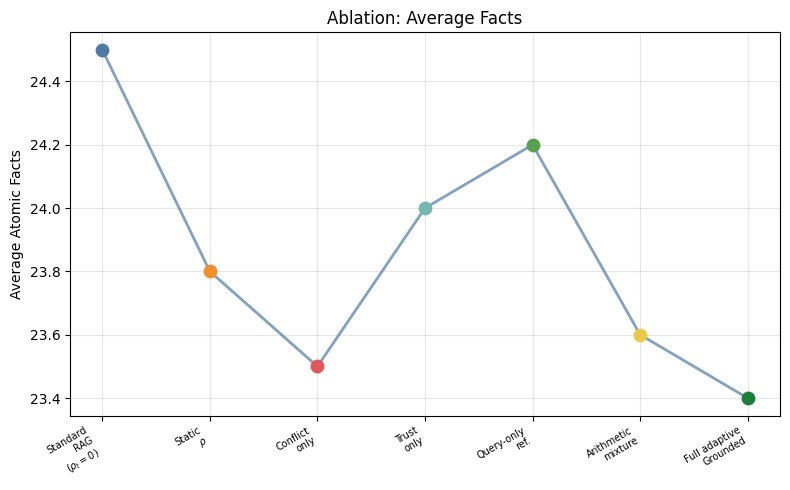}

\caption{Ablation study on factual content measured by the average number of atomic facts per generated response across different variants of the proposed method. We evaluate Standard RAG ($\rho_t=0$), Static $\rho$, Conflict-only, Trust-only, Query-only reference, arithmetic mixture, and the full adaptive Grounded Decoding model. The full adaptive model achieves the lowest number of hallucinated or unnecessary atomic facts while maintaining strong task performance, indicating improved factual conciseness and better grounded generation.}

\label{fig:ablation_avg_facts}
\end{figure}
\begin{figure}[t]
\centering
\includegraphics[width=0.95\columnwidth]{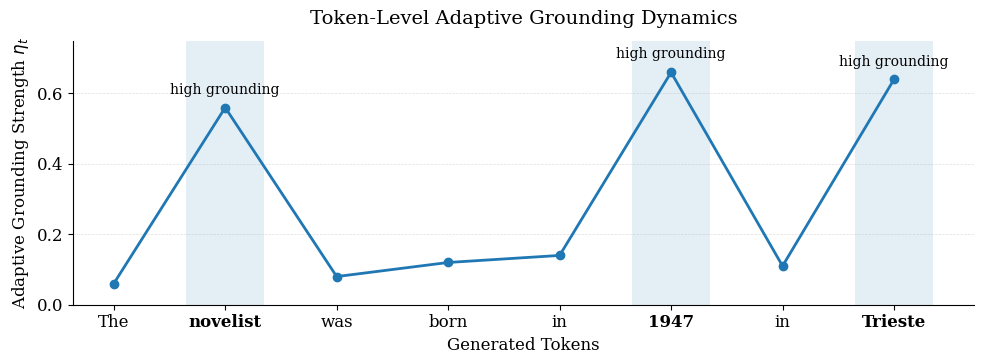}
\caption{Token-level adaptive grounding dynamics for a representative FActScore example. The grounding strength $\eta_t$ remains low for function words such as ``The'' and ``was'', while it increases sharply at factual spans including ``novelist'', ``1947'', and ``Trieste''. This behavior indicates that adaptive grounding selectively activates on content-bearing tokens that are more prone to factual errors, while remaining inactive on non-factual linguistic structure.}
\label{fig:adaptive_grounding_trajectory}
\end{figure}

\begin{figure}[t]
\centering
\includegraphics[width=\linewidth]{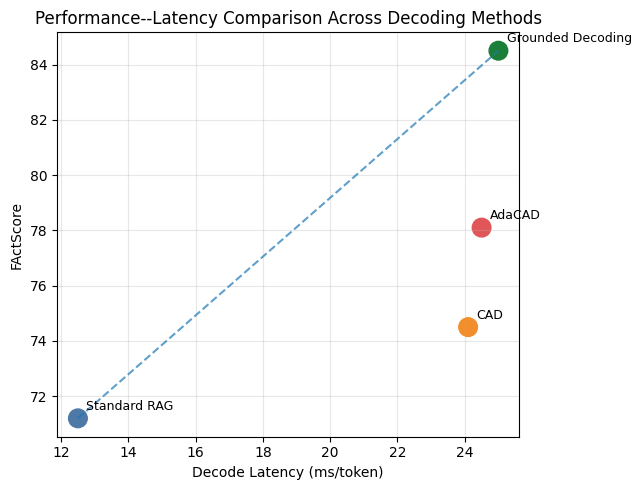}

\caption{Performance--Latency Comparison across decoding methods. We compare Standard RAG, CAD, AdaCAD, and Grounded Decoding in terms of decoding latency (ms/token) versus FActScore. While CAD and AdaCAD incur higher latency without substantial gains in factual accuracy, Grounded Decoding achieves the best trade-off, significantly improving FActScore while maintaining comparable computational cost to other advanced decoding methods.}

\label{fig:performance_latency_comparison}
\end{figure}

\begin{figure*}[t]
\centering
\includegraphics[width=\textwidth]{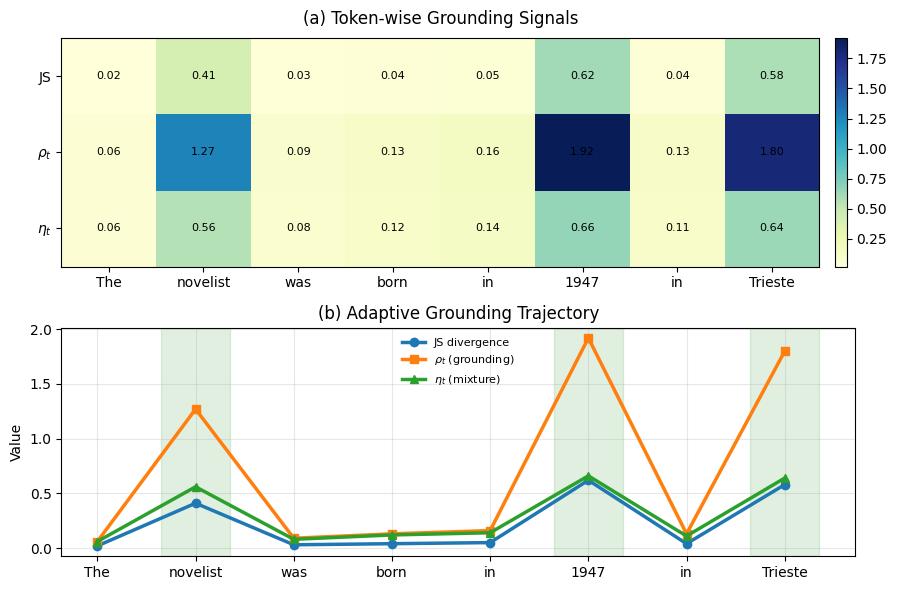}

\caption{Token-level grounding analysis of adaptive decoding. (a) The heatmap shows token-wise grounding signals including Jensen–Shannon divergence (JS), grounding weight $\rho_t$, and mixture coefficient $\eta_t$ across tokens in a generated sentence. Higher values are observed for factual tokens such as dates and named entities, indicating stronger grounding. (b) The line plot illustrates the adaptive grounding trajectory across tokens, showing how JS divergence, $\rho_t$, and $\eta_t$ evolve during generation. The model dynamically increases reliance on retrieved evidence for factual tokens while relying more on parametric generation for non-factual tokens, enabling fine-grained control of factual consistency during decoding.}

\label{fig:token_level_grounding}
\end{figure*}
\end{document}